\documentclass{ecai} 

\usepackage{latexsym}
\usepackage{amssymb}
\usepackage{amsmath}
\usepackage{amsthm}
\usepackage{booktabs}
\usepackage{enumitem}
\usepackage{graphicx}
\usepackage{color}
\usepackage{natbib} 
\usepackage{mathtools} 
\usepackage{booktabs} 
\usepackage{tikz} 
\usepackage[ruled,vlined,linesnumbered]{algorithm2e}
\usepackage{graphicx,subcaption}
\usepackage{multirow}

\newtheorem{theorem}{Theorem}
\newtheorem{lemma}[theorem]{Lemma}

\newcommand{\BibTeX}{B\kern-.05em{\sc i\kern-.025em b}\kern-.08em\TeX}
\DeclareMathOperator*{\argmax}{arg\,max}

\begin{document}


\begin{frontmatter}


\paperid{3} 


\title{Cost-constrained multi-label group feature selection using shadow features}


\author[A]{\fnms{Tomasz}~\snm{Klonecki}}
\author[A,B]{\fnms{Pawe{\l }}~\snm{Teisseyre}}
\author[C]{\fnms{Jaesung}~\snm{Lee}} 

\address[A]{Institute of Computer Science - Polish Academy of Sciences}
\address[B]{Faculty of Mathematics and Information - Science Warsaw University of Technology}
\address[C]{Department of Artificial Intelligence - Chung-Ang University}


\begin{abstract}
  We consider the problem of feature selection in multi-label classification, considering the costs assigned to groups of features. In this task, the goal is to select a subset of features that will be useful for predicting the label vector, but at the same time, the cost associated with the selected features will not exceed the assumed budget. Solving the problem is of great importance in medicine, where we may be interested in predicting various diseases based on groups of features. The groups may be associated with parameters obtained from a certain diagnostic test, such as a blood test.   
  Because diagnostic test costs can be very high, considering cost information when selecting relevant features becomes crucial to reducing the cost of making predictions.
  We focus on the feature selection method based on information theory. The proposed method consists of two steps. First, we select features sequentially while maximizing conditional mutual information until the budget is exhausted. In the second step, we select additional cost-free features, i.e., those coming from groups that have already been used in previous steps. Limiting the number of added features is possible using the stop rule based on the concept of so-called shadow features, which are randomized counterparts of the original ones.
  In contrast to existing approaches based on penalized criteria, in our method, we avoid the need for computationally demanding optimization of the penalty parameter. Experiments conducted on the MIMIC medical database show the effectiveness of the method, especially when the assumed budget is limited.
\end{abstract}

\end{frontmatter}


\section{Introduction}
Feature selection (FS) is one of the most important techniques in multi-label learning \citep{Kashefetal2018, BOGATINOVSKI2022}. 
The selection of relevant features leads to parsimonious models that are easier to interpret, which is important for their explainability while reducing the computational time of model training \citep{Molnar2022}.
Additionally, eliminating redundant and noisy features allows users to avoid the effect of overfitting and obtain a model with better predictive power \citep{Hastieetal2009}.

In the field of medical applications, limiting the number of features allows clinicians and patients to reduce the cost of making predictions \citep{Xuetal2014}. This is particularly important when obtaining feature values involves certain costs, such as diagnostic test costs.
Studies show that hospitals spend significant costs performing unnecessary diagnostic tests \citep{Vegting2012}. Moreover, unnecessary diagnostic tests or treatments, such as diagnostic X-rays, may cause negative effects \citep{HallBrenner2008}.

Costs can be associated not only with individual features but also with entire groups of features. For example, suppose that we are interested in predicting the co-occurrence of various diseases in a patient, which can be represented as a multi-label classification task \citep{Zufferey2015}. The model can be based on diagnostic test results, clinical variables, and administrative data. By incurring the cost of a general blood or urine test, we gain access to a whole group of different parameters that serve as features in classification tasks. Similarly, data collected during a medical interview (e.g., information about past diseases or medications used) may constitute one group. Moreover, a group of features may consist of various statistics based on measurements of clinical parameters. For example, based on daily blood pressure measurement, statistics such as average/maximum value or number of blood pressure spikes can be calculated. 

In this work, we consider the cost-constrained group feature selection problem \citep{Pacliketal2002, KloneckiTeisseyreLeeMDAI2023}, in which the goal is to select a subset of features that the label vector is dependent on while limiting the budget to obtain feature values. Unlike most existing approaches \citep{BOLONCANEDO2014, ZHOU2016, jagdhuber2020cost, TeisseyreKlonecki2021, long2021cost}, we assume that costs are assigned to groups of features, i.e., by incurring the cost of a group, we gain access to all features belonging to a given group. To our knowledge, this problem has not yet been considered for multi-label classification.

We focus on model agnostic approaches based on information theory measures, such as Mutual Information (MI) and Conditional MI (CMI) \citep{Cover2006}. This approach is useful as a data pre-processing step performed before training the final classification model.
A natural approach employed in existing work is to use a Sequential Forward Selection (SFS), taking into account the penalty for the cost of the added feature \citep{BOLONCANEDO2014}.
In these methods, in each step, the algorithm adds a feature that is the most informative in the context of the features selected in the previous steps. The informativeness of a candidate feature is measured using CMI. Additionally, the algorithm subtracts a penalty for its cost from the CMI. The penalty can be the cost multiplied by the penalty parameter, which controls the balance between the importance of the candidate feature and its cost. 
In practice, the value of this parameter has a great impact on the selection of features and the performance of the classifier. Too small parameter value leads to omitting cost information and selecting too expensive features. On the other hand, if the value of the penalty parameter is too high, the cheapest features are selected, which do not have to be associated with the predicted labels.

In our work, we present an alternative approach to avoid the need to choose the penalty parameter.
The proposed method consists of two steps. In the first step, the algorithm uses the SFS scheme, ignoring the cost information. The algorithm chooses features until the assumed budget is exceeded. Note that among the remaining features, there are those whose cost is zero. These are the features of the groups that have been used in the previous steps. Therefore, in the second step, the algorithm selects additional features from among the zero-cost ones. To avoid including too many features and, consequently, overfitting, determining the moment of stopping becomes crucial. Our idea is to use so-called shadow variables generated by randomizing the original features \citep{Kursa2011}. Selecting the first shadow feature or some fraction of shadow features indicates that the algorithm should stop adding further ones.
Experiments conducted on the MIMIC medical database indicate that the proposed method performs effectively compared to methods based on the penalized SFS scheme and traditional FS methods. When the assumed budget is small, it achieves significantly greater accuracy than competing methods.

The paper is structured as follows. In Section 2, we discuss the most related methods. In Section 3, we formulate the problem of multi-label cost-constrained group feature selection based on the concepts of information theory, present the proposed approach, and the idea of using shadow features. Section 4 illustrates how the algorithm works and presents the results of experiments performed on a large clinical dataset MIMIC \citep{mimic2-dataset}.

\section{Related work}
SFS methods based on information theory are popular for selecting features for classification models \citep{Brown2012}. They are mainly based on CMI, which measures the usefulness of adding the candidate feature to previously selected ones. Because estimating CMI is challenging, most studies considered certain score functions that replace CMI, but are easier to estimate. 
An example is the popular CIFE criterion \citep{Lin2006}, which we obtain by using the M{\"o}bius representation for CMI \citep{Meyer2008} and then considering only the first two terms: MI between the class variable and the candidate feature and interaction information including the class variable, the candidate feature and one of the already selected features.
Another frequently used approach adopted in our work is to use the lower bound of the CMI. This leads to the Joint MI criterion \citep{Yang1999}, which contains CMI between the class variable and the candidate feature, and conditioning is performed on one of the features selected in the previous steps.

Many modifications have also been proposed for multi-label classification, among which \citep{LeeKim2012, LeeKim2017, seo2019generalized, zhang2019distinguishing} are representative examples.
Note, however, that the multi-label case is more interesting but also more demanding.
First, CMI estimation becomes more difficult due to the dimension of the label vector. Secondly, the issue of the existence of interactions plays an important role: in addition to interactions between multiple variables, there are also interactions involving multiple labels \citep{TEISSEYRELEE2023}.

Regarding cost-constrained FS, the approaches considered modify the traditional SFS methods by adding a penalty for the cost of the candidate variable \citep{BOLONCANEDO2014, TeisseyreKlonecki2021, KloneckiTeisseyreLeePR2023}. 
In addition to approaches based on information theory, modifications of traditional methods taking into account costs were considered, among which modification of random forest importance measures is a representative example \cite{ZHOU2016}.

In the methods mentioned above, the costs are assigned to individual features. 
For the problem where the cost is assigned on the basis of each feature group was considered 
in the works of \cite{Pacliketal2002, KloneckiTeisseyreLeeMDAI2023} where the former is related to the segmentation of backscatter images for product analysis and the latter is about the multi-morbidity prediction, respectively.
A common drawback of both approaches is that they are computationally intensive. In \cite{Pacliketal2002}, the performance on validation sets of a subsequent classifier is used as a scoring function. This requires learning the model several times, which can be prohibitive for large datasets. In \cite{KloneckiTeisseyreLeeMDAI2023}, optimizing the cost parameter requires running the FS algorithm repeatedly.
To address the above issues, we propose a method independent of the classification model and that does not require tuning additional penalty parameters. It uses shadow features obtained from the original features by random permutation of values. A similar idea has already been used in work on FS, e.g., \cite{Kursa2011} combined shadow features with selection methods based on random forests, while \cite{TEISSEYRELEE2023} proposed to separate relevant and irrelevant features using the threshold determined on the basis of permutation distribution.

Finally, let us note that the problem we are considering is different from the problem of group FS, which is studied extensively in the statistical literature (the group lasso method \citep{YuanLin2006} is an example).
In a group FS problem, the task is to select entire groups, while in the problem we are considering, the goal is to select individual features without exceeding a limited budget.

\section{Multi-label cost-constrained group feature selection}
\subsection{Problem statement}
Let $X=(X_1,\ldots,X_p)$ be a vector of features and $Y=(Y_1,\ldots,Y_q)$ be a target vector containing binary variables describing whether the given label is assigned to the instance. More precisely, $Y_i=1$ indicates that $i$-th label is relevant; for example, $i$-th disease is present in a given patient. Moreover, we denote by $F=\{1,\ldots,p\}$ the set of indices of the features. We assume that there are $K$ disjoint groups of features $G_1,\ldots,G_K$, such that $F=G_1\cup\ldots\cup G_K$ where $G_i \cap G_j = \emptyset$ $(i \neq j)$ and costs $c(G_1),\ldots,c(G_K)$ are associated with the groups. If we acquire the value of one feature of the group, then the values of the remaining features of the same group are obtained for free.  The cost associated with the given subset of features $S \subseteq F$ can be written as 
    \begin{equation}
        \label{Cost_group}
         c(S)=\sum_{k=1}^{K}c(G_k)I\{\exists j\in S: j\in G_k\}.
    \end{equation}

The basic measure that we use to measure the strength of the dependence between the label vector $Y$ and the feature vector $X$ is MI defined as $M(Y,X)=H(Y)-H(Y|X)$, where $H(Y), H(Y|X)$ are the entropy and the conditional entropy, respectively.   
Within the information-theoretic framework, the grouped cost-constrained FS problem can be stated as 
    \begin{equation}
        \label{Problem_cost_sensitive}
        S_{opt} = \argmax_{S:c(S)\leq B}M(Y,X_S),
    \end{equation} 
where $B$ is a user specified maximal admissible budget and $X_S$ is a vector corresponding to subset $S\subseteq F$. Intuitively, we try to find a set of features that is maximally dependent on the label vector such that the cost of the features used does not exceed the given budget $B$.

\subsection{Sequential Forward Selection (SFS) with cost penalty}
Solving (\ref{Problem_cost_sensitive}) requires an infeasible computational cost of an exhaustive search on feature subsets. 
Therefore, a natural method, adopted from several studies, is to use a sequential forward algorithm that, in each step, adds a candidate feature $X_k$, which is the most informative in the context of already selected features described in $S\subseteq F$. 
The CMI $M(X_k,Y|X_S)=H(Y|X_S)-H(Y|X_k,X_S)$ is a natural score function, which allows to assess the relevance of a candidate feature $X_k$ given the set $X_S$ of already selected features. The CMI is equal to the increase of MI, associated with adding a candidate feature $M(X_k,Y|X_S)=M(Y,X_{S\cup\{k\}})-M(Y,X_{S})$ \citep{Cover2006, Brown2012} and thus directly corresponds to MI in (\ref{Problem_cost_sensitive}). 

Recently, a general algorithm was proposed that considers the group structure of costs \citep{KloneckiTeisseyreLeeMDAI2023}, which works as follows. 
The algorithm starts from an empty set $S=\{\emptyset\}$ and in each step adds a candidate feature $S\leftarrow S\cup\{k_{\text{opt}}\}$ such that
\begin{equation}
    \label{Method1}
    k_{\text{opt}}=\argmax_{k\in F\setminus S}[M(X_k,Y|X_S)-\lambda c(k,S)],
\end{equation}
where $\lambda>0$ is a parameter and $c(k,S)$ is a penalty for the cost of adding feature $X_k$. Unlike the standard cost-constrained problem, the cost of adding $X_k$ is not constant. It depends on whether, in the previous steps, the algorithm selected features from the group to which the feature $X_k$ belongs. If the algorithm has already selected features from the considered group $G$, then the cost is zero. Otherwise, the selection of a feature incurs the cost assigned to the group, which can be written as
 \[
c(k,S)=
\begin{cases}
    0\text{ if } k\in G \textrm{ and } \exists j\in S: j\in G \\
    c(G) \text{ if } k\in G \textrm{ and } \nexists j\in S: j\in G. \\
\end{cases}
\]

The algorithm stops adding new features in (\ref{Method1}), when $c(S)>B$.
Parameter $\lambda>0$ controls a trade-off between feature relevance and the cost. Note that $\lambda=0$ corresponds to a traditional FS method that does not consider cost. In general, choosing the optimal value of $\lambda$ is a challenging task. For too small $\lambda$ values, the costs will be ignored, which may lead to the selection of a set of features that do not fit the assumed budget. On the other hand, for too high $\lambda$ values, the algorithm will ignore the relevance of the features and choose the cheapest. This can be a disadvantage if cheap features are not informative or redundant. Moreover, in this case, the method will tend to select features that come from groups from which any feature was selected, regardless of the significance of the included feature.
The previously considered methods of optimizing $\lambda$ require running the algorithm for different values from a certain grid $\lambda_{\min},\ldots,\lambda_{\max}$ and then choosing a value related to maximization of a certain criterion, such as a cumulative sum of relevance terms \citep{TeisseyreKlonecki2021, KloneckiTeisseyreLeePR2023}. Unfortunately, this approach is computationally expensive and becomes prohibitive if there is a certain constraint on the execution time.

\subsection{Proposed method}
To avoid the computationally demanding problem of $\lambda$ optimization, we propose a new, simple but effective approach based on the so-called shadow features.
The proposed method consists of two steps. In the first step, the algorithm performs a traditional FS process, i.e., method (\ref{Method1}) with $\lambda=0$, which yields a subset of features $S_1$. Importantly, among the remaining features $F\setminus S_1$, two subsets can be distinguished, which we denote by $U$ and $V$. The subset $U=\{j_1,\ldots,j_u\}$ of size $u$ contains features of zero cost, coming from groups of which at least one feature has already been selected. The second subset $V$ consists of features from groups that have not been used before, and therefore $V$ is no longer useful to the algorithm because adding any feature from the set $V$ to $S_1$ will exceed the assumed budget $B$.
Because the features in $U$ have zero cost, they may still be useful to improve the relevance of the selected feature subset on the target vectors.
The main idea of the algorithm is to add some fraction of features from set $U$ to the set $S_1$ to improve the quality of the final subset.

The natural way is to add $X_k$, $k\in U$ that maximizes $M(X_k,Y|X_{S_1})$ and continue adding the remaining features of $U$ analogously. However, the question arises when we should stop adding new features. The set $U$ may be large, and adding all the features of $U$ may lead to the inclusion of too many features and lead to overfitting of subsequent classification models.
To address this problem, we propose to use so-called shadow features to determine the stopping rule. Shadow features are generated by permuting values of the original features. We create shadow features $X_{j_1}^*,\ldots,X_{j_u}^*$ corresponding to features from the set $U$. Observe that the shadow features are independent of both $Y$ and $X$, and thus, they are irrelevant in predicting $Y$.
Moreover, we have the following fact, which formally justifies that the shadow feature is non-informative in the context of any subset of selected features. 
\begin{lemma}
\label{Lemma1}
Let $X_k^*$ be the shadow feature. Then we have $M(X_k^*,Y|X_S)=0$, for any $S\subseteq F$.
\end{lemma}
\begin{proof}
Note that $M(X_k^*,(Y,X_S))=0$ is equivalent to 
\begin{equation}
\label{L1_E1}
M(X_k^*,X_S)+M(X_k^*,Y|X_S),
\end{equation}
which follows from the chain rule for the MI \citep{Cover2006}.
By definition of the shadow feature, we have $M(X_k^*,(Y,X))=0$. Observe that independence $X_k^*\perp (Y,X)$ implies $X_k^*\perp (Y,X_S)$, for any subset $S\subseteq F$, which in turn implies $M(X_k^*,(Y,X_S))=0$. Hence, in view of (\ref{L1_E1}), it means $M(X_k^*,Y|X_S)=0$.
\end{proof}

The advantage of shadow features is that the marginal distribution remains unchanged, i.e., $P(X_{j_k}=x)=P(X_{j_k}^*=x)$. Thus, it can be regarded that the shadow features mimic the original features, but unlike the original features, they are always non-informative.
In the second step of our procedure, we add the subset $S_2\subseteq U$ to the feature indices set $S_1$. 
We initialize $S_2$ with $S_2=\emptyset$ and then, in each step, we add feature $S_2\leftarrow S_2\cup\{k_{\text{opt}}\}$ such that
    \begin{equation}
        \label{Method2}
        k_{\text{opt}}=\argmax_{k\in A\setminus S_2}M(X_k,Y|X_{S_1\cup S_2}).
    \end{equation}

In addition define
\[M_{\max}^*:=\max_{k\in A\setminus S_2}M(X_k^*,Y|X_{S_1\cup S_2}),\]
being the maximal CMI, computed for shadow features.
Obviously, in view of Lemma \ref{Lemma1}, the theoretical value of $M_{\max}^*$ is $0$, however in practice, we replace the MI by some estimator, computed on observed training data and thus the empirical counterpart of $M_{\max}^*$ can take non-zero values. We refer to the discussion in Section
\ref{Sec:Estimation of Conditional Mutual Information} on how to estimate MI and CMI.

We stop adding new features from $A$ when 
    \begin{equation}
        \label{StopRule}
        M_{\max}^*> M(X_{k_{\text{opt}}},Y|X_{S_1\cup S_2}),
    \end{equation}
which means that one of the shadow features is more informative in the context of already selected features than $X_{k_{\text{opt}}}$.
The stopping rule can be modified; for example, the algorithm may stop adding new features only if we choose a certain fraction (e.g. 5\%) of shadow features.
Algorithm~\ref{alg:alg1} describes the whole procedure.

\begin{algorithm}[!t]
            \caption{Proposed method}
            
            \DontPrintSemicolon
            \SetAlgoLined
            \SetNoFillComment
            \LinesNotNumbered 
            
            \SetKwInOut{Parameters}{Parameters}
            \SetKwInOut{Input}{Input}
            \SetKwInOut{Output}{Output}
            \SetKwInOut{Repeat}{Repeat}
            \SetKwInOut{Until}{until}
            
            \Input{$Y$, $X_1,\ldots,X_p$ $B$, $c(G_1),\ldots,c(G_K)$}

            \# Step 1:
            
            $S_1=\emptyset$ \#initialization

              \While{$c(S_1)\leq B$}{
            
            $k_{\text{opt}}=\arg\max_{k\in F\setminus S_1}M(X_k,Y|X_{S_1}),$
            
            $S_1\leftarrow S_1\cup\{k_{\text{opt}}\}$.
            
            }
            \# Step 2:

            $S_2:=\emptyset$ \#initialization
            
            $M_{\max}^*:=0$ \#initialization

            $U:=\{j_1,\ldots,j_u\}$ \# a set of zero-cost features

            \While{$U\setminus S_2\neq\emptyset$}{

                 $k_{\text{opt}}=\argmax_{k\in U\setminus {S_2}}M(X_k,Y|X_{S_1\cup S_2})$.

                $S_2\leftarrow S_2\cup\{k_{\text{opt}}\}$

                $M_{\max}^*:=\max_{k\in U\setminus S_2}M(X_k^*,Y|X_{S_1\cup S_2})$

                \If{$M_{\max}^*> M(X_{k_{\text{opt}}},Y|X_{S_1\cup S_2})$}{
                    {\bf break}
                }
            }

        $S_{\text{opt}}=S_1\cup S_2$

            \Output{$S_{\text{opt}}$}
            
            \label{alg:alg1}
        \end{algorithm}   

\subsection{From conditional mutual information to low-dimensional score functions}
\label{Sec:Estimation of Conditional Mutual Information}
The methods described above are very general and based on CMI. In practice, CMI estimation is demanding, especially for small training data, and in the case of large dimension of the set $S$ on which we condition and large dimension of $Y$  \citep{Paninski2003, Belghazietal2018}.
In the case of discrete variables, a popular method is to use plug-in estimators, which involve replacing the probabilities in the entropy definition with their empirical equivalents. However, this requires the estimation of multivariate probability distributions. In the case of continuous variables, the problem becomes even more difficult.
One possibility is to employ MI estimation methods using variational inference, such as the MINE (Mutual Information Network Estimation) \citep{Belghazietal2018} algorithm. This method is based on the so-called Donsker-Varadhan \citep{DonskerVaradhan1975} representation for MI and uses neural networks to optimize the appropriate risk function. The disadvantage of this approach is the significant computational cost, which makes it less useful in the context of fast, model-free selection methods.
Another possibility, used by many other authors \citep{Brown2012}, is to replace high-dimensional measures with lower-order terms, i.e., terms that depend on probabilities covering at most a few variables. This can be done using various techniques, such as lower MI limits or the use of the M{\"o}bius representation \citep{Meyer2008}.
Since the CMI is calculated multiple times in our approach, we decided to use an approach based on the lower bound of the MI. Specifically, 
$
M(Y,X_k|X_S)=M(Y,X_{S\cup\{k\}})-M(Y,X_{S})\propto M(Y,X_{S\cup\{k\}}), 
$
where the proportionality $\propto$ follows from the fact that $M(Y,X_{S})$ does not depend on the candidate feature $X_k$ and therefore can be omitted from the scoring function.
It is a well-known fact \citep{Cover2006} that conditioning on a larger number of features reduces entropy (information can’t hurt), which in turn implies that
\begin{eqnarray*}
&&
M(Y,X_{S\cup\{k\}})=H(Y)-H(Y|X_{S\cup\{k\}})
\cr
&&
\geq H(Y)-H(Y|X_{A})=M(Y,X_{A}),
\end{eqnarray*}
for any subset $A\subseteq S\cup\{k\}$.  
Similarly, we have
\begin{eqnarray*}
&&
M(Y,X_{S\cup\{k\}})=H(X_{S\cup\{k\}})-H(X_{S\cup\{k\}}|Y)
\cr
&&
\geq H(X_{S\cup\{k\}})-H(X_{S\cup\{k\}}|Y_B)=M(Y_B,X_{S\cup\{k\}}),
\end{eqnarray*}
where $B\subseteq\{1,\ldots,q\}$\footnote{Here, we use $B$ as an arbitrary set of index corresponding to labels $Y$. Please do not confuse this with the budget $B$.}.
Applying the above inequalities and averaging over all subsets $A,B$ of sizes $a$ and $b$, respectively, we obtain a lower bound on MI, 
\[
M(Y,X_{S\cup\{k\}})\geq \frac{1}{\binom{p}{a}}\frac{1}{\binom{q}{b}}\sum_{A\subseteq S\cup\{k\}}\sum_{B\subseteq \{1,\ldots,q\}}M(Y_B,X_{A}),
\]
where $X_A, Y_B$ are subvectors of $X,Y$, corresponding to sets $A,B$, respectively.
The above idea has already been used in related papers on FS (\cite{Meyer2008, KloneckiTeisseyreLeePR2023}).
The selection of appropriate values of $a$ and $b$ leads to criteria that are proportional to the lower limit ($>\propto $) of MI and can be used instead of CMI in formulas (\ref{Method1}), (\ref{Method2}) and (\ref{StopRule}). For example, for $a=1, b=2$ we obtain
\begin{eqnarray*}
&&
M(Y,X_{S\cup\{k\}})>\propto \sum_{l=1}^{q}\sum_{j\in S}M(Y_l,X_k|X_j)
\cr
&&
=\sum_{l=1}^{q}\sum_{j\in S}\underbrace{M(Y_l,X_k)}_{\text{(i)}}+\underbrace{II(Y_l,X_k,X_j)}_{\text{(ii)}}],
\end{eqnarray*}
which is a generalization of popular JMI criterion \citep{Yang1999} to the multi-label case.
Importantly, the term (i) corresponds to the marginal dependence between the label $Y_l$ and the candidate feature $X_k$. The second term (ii) contains interaction information (II) defined as $II(Y_l,X_k,X_j)=M(Y_l,X_k|X_j)-M(Y_l,X_k)$, which describes the strength of the interaction between features $X_k$ and $X_l$ in predicting the label $Y_l$ (\cite{Han1980, Lee2015mutual, MielniczukTeisseyre2018}). Including the second term is useful because it may happen that the feature $X_k$ is marginally independent of $Y_l$, i.e., $M(Y_l,X_k)=0$, but it interacts with feature $X_j$, i.e., $II(Y_l,X_k,X_j)>0$. The well-known example of such a situation is $Y_l=XOR(X_k,X_j)=I(X_k\neq X_j)$, where $X_j$ and $X_k$ are binary, independent features.
Moreover, $II$ can take negative values, indicating redundancy related to the feature $X_k$. For example, when $X_k=X_j$, then $II(Y_l,X_k,X_j)=-M(Y_l,X_k)<0$.

By choosing larger $a$ and $b$, we can account for higher-order interactions between features and labels. However, this involves increased computational costs and the estimation of higher-order terms. In experiments, we used $a=b=1$.
In practice, the choice of a and b should depend on the number of observations (the more observations, the more accurately we can estimate higher-order terms) and computational resources (the larger $a$ and $b$, the more components we need to estimate).

\section{Experiments}

\begin{table*}[t!]
    \centering
    \caption{Illustrative example - results.}
    \begin{tabular}{cccccc}
        \hline
        \multicolumn{1}{l}{} & \textbf{Method} & Step = 1 & Step = 2 & Step = 3 & Hamming loss\\ \hline
        \multirow{2}{*}{$B = 1$} & Traditional & $\{X_1\}$ & $-$ & $-$ & $0.383$\\
        & Proposed & $\{X_1\}$ & $\{X_1, X_2\}$ & $\{X_1, X_2, X_3\}$ & $0.257$ \\ 
        \hline
        \multirow{2}{*}{$B = 2$} & Traditional & $\{X_1\}$ & $\{X_1, X_4\}$ & $-$ & $0.302$ \\
        & Proposed & $\{X_1\}$ & $\{X_1, X_4\}$ & $\{X_1, X_3, X_4\}$ & $0.222$ \\ 
        \hline
        \multirow{2}{*}{$B = 3$} & Traditional & $\{X_1\}$ & $\{X_1, X_4\}$ & $\{X_1, X_4, X_5\}$ & $0.203$\\
        & Proposed & $\{X_1\}$ & $\{X_1, X_4\}$ & $\{X_1, X_4, X_5\}$ & $0.203$ \\ 
        \hline
    \end{tabular}
    \label{tab:ilustrative examples}
\end{table*}

The primary objective of the experiments was to evaluate the effectiveness of the proposed cost-constrained method. For comparison, we used methods based on SFS with a cost penalty $\lambda$, including the case of $\lambda=0$, corresponding to traditional FS. Moreover, we considered the cases $\lambda=\lambda_{\max}$ and $\lambda=0.5\cdot\lambda_{\max}$, where $\lambda_{\max}$ is chosen in such a way that in the first step, the feature with a lower cost is always selected before the feature with a higher cost, regardless of its relevance. Our goal is to show that a strategy based on selecting features with the lowest cost does not necessarily lead to high model accuracy. In turn, the value of $0.5\cdot\lambda_{\max}$ corresponds to an intermediate situation in which both the cost and the correlation with the label vector influence the selection of the candidate feature.

\subsection{Illustrative example}
To give a deeper insight into the method proposed in Algorithm~\ref{alg:alg1}, let us consider a simple illustrative example based on a synthetic dataset.
It shows in what situations we can expect that our method will work more effectively than the traditional FS method that does not take into account cost information.

The data generation process is shown in Figure \ref{fig:illustrative_schema}.
Specifically, consider five features $X_1, X_2, X_3, X_4$, and $X_5$, generated independently from the normal distribution $N(0,1)$. 
Additionally, there are three binary target variables, $Y_1, Y_2$, and $Y_3$, generated directly from $X_1, X_4$ and $X_5$ as $P(Y_1=1|X)=\sigma(3 X_1)$, $P(Y_2=1|X)=\sigma(2 X_4)$ and $P(Y_3=1|X)=\sigma(X_5)$, where $\sigma(s)=(1+\exp(-s))^{-1}$ represents a sigmoid activation function.
Furthermore, the features $X_2$ and $X_3$ are copies of $X_4$ and $X_5$, respectively, except that $\rho = 20\%$ of their values have been randomly permuted. This alteration makes $X_2$ and $X_3$ less informative than their counterparts. We consider the following feature groups: $G_1 = \{1, 2, 3 \}$, $G_2 =\{4 \}$ and $G_3 =\{5 \}$ and each group cost is equal to $1$. 
Figure~\ref{fig:illustrative_mi} shows the significance of individual features, measured as the MI of a given feature with a label, summed over all labels.
The information needed to effectively predict the label vector is contained in all groups, which requires a cost of three. On the other hand, the variables $X_2$ and $X_3$ from group $G_1$ can replace groups $G_2$ and $G_3$. This way, we can build a slightly less accurate model while incurring a much lower cost of one.

\begin{figure}[ht!]
    \begin{subfigure}[t]{\linewidth}
      \centering
      \includegraphics[width=\textwidth]{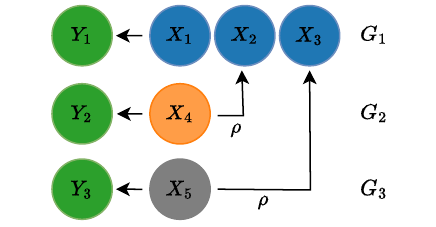}
      \caption{Data generation process.}
      \label{fig:illustrative_schema}
      \vspace*{0.4cm}
    \end{subfigure}
    \begin{subfigure}[t]{\linewidth}
      \centering
      \includegraphics[width=\textwidth]{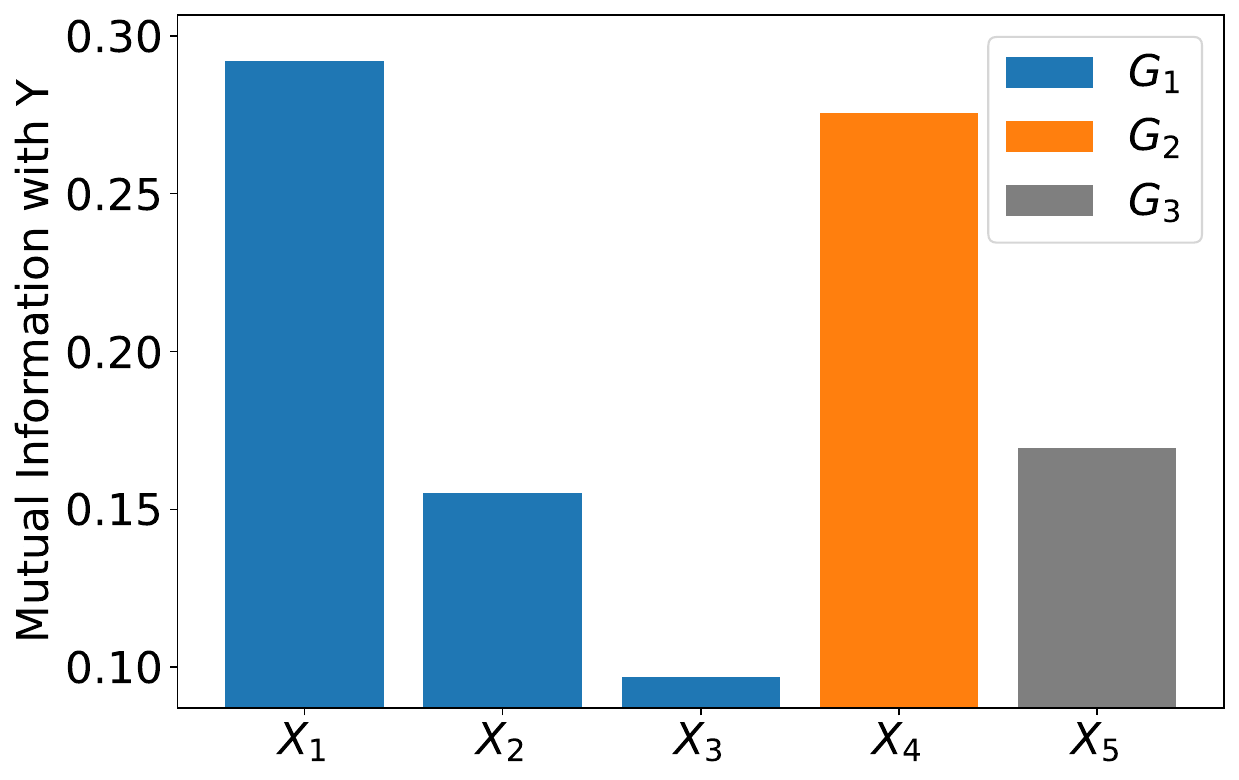}
      \caption{Relevance of the considered features.}
      \label{fig:illustrative_mi}
      \vspace*{0.4cm}
    \end{subfigure}
    \caption{Illustrative example.}
\end{figure}
\vspace*{0.5cm}

In Table~\ref{tab:ilustrative examples}, we present the features chosen in consecutive steps of the method, along with the Hamming loss of the specific model trained on the selected feature subset. 
The FS process was performed for three different budget values $B: 1,2,3$.
For $B=1$, the traditional method selects only one feature $X_1$, which is the most informative.
On the other hand, the proposed method selects all features of the group $G_1$, resulting in a lower Hamming loss. It should be noted that the feature $X_1$ is chosen in the first step, while the features $X_2,X_3$ are added in the second step as zero-cost features.
For a larger budget $B=2$, the traditional method selects two variables $X_1,X_4$, while the proposed method additionally includes $X_2$, which results in a reduction of the Hamming loss from $0.302$ to $0.222$ 
As expected, for the largest budget $B=3$, both methods select the same features.
In summary, the proposed method allows adding additional variables that do not increase the cost and improve the classification quality of the model.

\begin{table}
    \centering
    \caption{Data statistics.}
    \begin{tabular}{lc}
        \toprule
        Description & Values \\
        \midrule
        number of observations (patients)& $19,773$ \\
        number of features & $305$ \\
        number of feature groups & $87$ \\
        number of labels (diseases) & $10$ \\
        min/avg/max  diseases for a patient & $1/2.4/8$ \\
        min/avg/max size of the group & $1/3.466/9$ \\
        min/avg/max cost of the group & $1/7.73/53.5$ \\
    \bottomrule
\end{tabular}
\label{tab:mimic-summary-statistics2}
\end{table}

\begin{table}
    \caption{Feature groups.}
    \label{tab:groups}
    \centering
    \footnotesize
    \begin{tabular}{cccc}
        \toprule
        Group Name & Description & Cost & \# Features\\ 
        \midrule
        A & Administrative data & $1.0$ & $9$ \\
        NBP & Blood pressure & $3.0$ & $8$ \\
        RL & Right lung sounds & $9.0$ & $4$ \\
        UN & Urea nitrogen in serum & $12.0$ & $4$ \\
        VR & Patient verbal response & $2.0$ & $3$ \\
        HR & Heart rate & $1.5$ & $4$ \\
        P & Blood test (platalets) & $2.0$ & $4$\\
        \bottomrule
    \end{tabular}
\end{table}

\begin{table*}[!t]
\centering
\caption{MIMIC dataset - miscellaneous metrics.}
\label{tab:mimic_results_hamming_loss}
    \begin{tabular}{cccccc}
        \toprule
        \textbf{Budget} & \textbf{Metric} & Proposed & Traditional & \textbf{$SFS \: (\lambda_{\max})$} & \textbf{$SFS \: (0.5 \: \lambda_{\max})$} \\
        \midrule
        $10$ & Hamming Loss & $\textbf{0.214} \pm \textbf{0.001}$ & $-$ & $0.217 \pm 0.004$ & $0.217 \pm 0.004$ \\
        $25$ & & $\textbf{0.196} \pm \textbf{0.002}$ & $0.202 \pm 0.003$ & $0.212 \pm 0.002$ & $0.210 \pm 0.002$ \\
        $120$ & & $\textbf{0.191} \pm \textbf{0.001}$ & $0.191 \pm 0.003$ & $0.196 \pm 0.002$ & $0.194 \pm 0.001$ \\
        \midrule
        $10$ & Ranking Loss & $\textbf{0.228} \pm \textbf{0.004}$ & $-$ & $0.232 \pm 0.004$ & $0.232 \pm 0.003$ \\
        $25$ & & $\textbf{0.205} \pm \textbf{0.004}$ & $0.213 \pm 0.005$ & $0.218 \pm 0.004$ & $0.220 \pm 0.004$ \\
        $120$ & & $\textbf{0.195} \pm \textbf{0.003}$ & $0.195 \pm 0.004$ & $0.201 \pm 0.003$ & $0.200 \pm 0.002$ \\
        \midrule
        $10$ & Coverage Error & $\textbf{4.998} \pm \textbf{0.059}$ & $-$ & $5.033 \pm 0.039$ & $5.026 \pm 0.041$ \\
        $25$ & & $\textbf{4.811} \pm \textbf{0.066}$ & $4.889 \pm 0.069$ & $4.898 \pm 0.041$ & $4.914 \pm 0.044$ \\
        $120$ & & $\textbf{4.697} \pm \textbf{0.037}$ & $4.709 \pm 0.054$ & $4.797 \pm 0.087$ & $4.786 \pm 0.066$ \\
        \midrule
        $10$ & Zero One Loss & $\textbf{0.914} \pm \textbf{0.002}$ & $-$ & $0.918 \pm 0.007$ & $0.917 \pm 0.005$ \\
        $25$ & & $\textbf{0.895} \pm \textbf{0.011}$ & $0.908 \pm 0.011$ & $0.909 \pm 0.005$ & $0.908 \pm 0.003$ \\
        $120$ & & $0.885 \pm 0.002$ & $\textbf{0.884} \pm \textbf{0.005}$ & $0.889 \pm 0.006$ & $0.889 \pm 0.003$ \\
        \midrule
        $10$ & Accuracy & $\textbf{0.086} \pm \textbf{0.002}$ & $-$ & $0.082 \pm 0.007$ & $0.083 \pm 0.005$ \\
        $25$ & & $\textbf{0.105} \pm \textbf{0.011}$ & $0.092 \pm 0.011$ & $0.091 \pm 0.005$ & $0.092 \pm 0.003$ \\
        $120$ & & $0.115 \pm 0.002$ & $\textbf{0.116} \pm \textbf{0.005}$ & $0.111 \pm 0.006$ & $0.111 \pm 0.003$ \\
        \midrule
        $10$ & F1 (micro avg) & $\textbf{0.461} \pm \textbf{0.002}$ & $-$ & $0.448 \pm 0.004$ & $0.443 \pm 0.005$ \\
        $25$ & & $\textbf{0.516} \pm \textbf{0.008}$ & $0.494 \pm 0.014$ & $0.478 \pm 0.007$ & $0.475 \pm 0.005$ \\
        $120$ & & $0.531 \pm 0.003$ & $\textbf{0.535} \pm \textbf{0.006}$ & $0.526 \pm 0.009$ & $0.533 \pm 0.003$ \\
        \midrule
        $10$ & AUC (micro avg) & $\textbf{0.768} \pm \textbf{0.003}$ & $-$ & $0.763 \pm 0.003$ & $0.762 \pm 0.003$ \\
        $25$ & & $\textbf{0.792} \pm \textbf{0.003}$ & $0.782 \pm 0.005$ & $0.778 \pm 0.003$ & $0.774 \pm 0.003$ \\
        $120$ & & $\textbf{0.803} \pm \textbf{0.002}$ & $0.802 \pm 0.002$ & $0.795 \pm 0.003$ & $0.798 \pm 0.001$ \\
        \bottomrule
    \end{tabular}
\end{table*}

\begin{figure*}
    \centering
    \begin{tabular}{c c c}
      \includegraphics[width=0.33\textwidth]{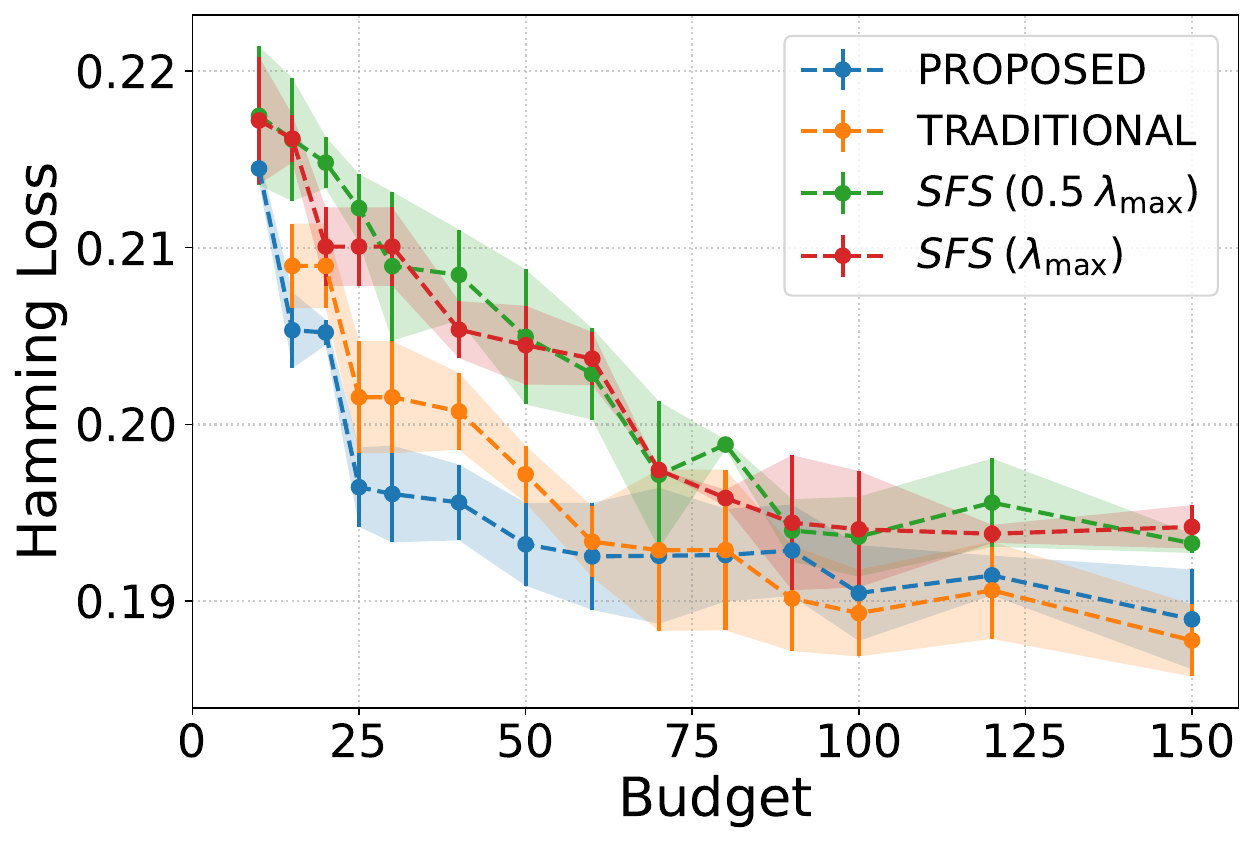}
      \includegraphics[width=0.33\textwidth]{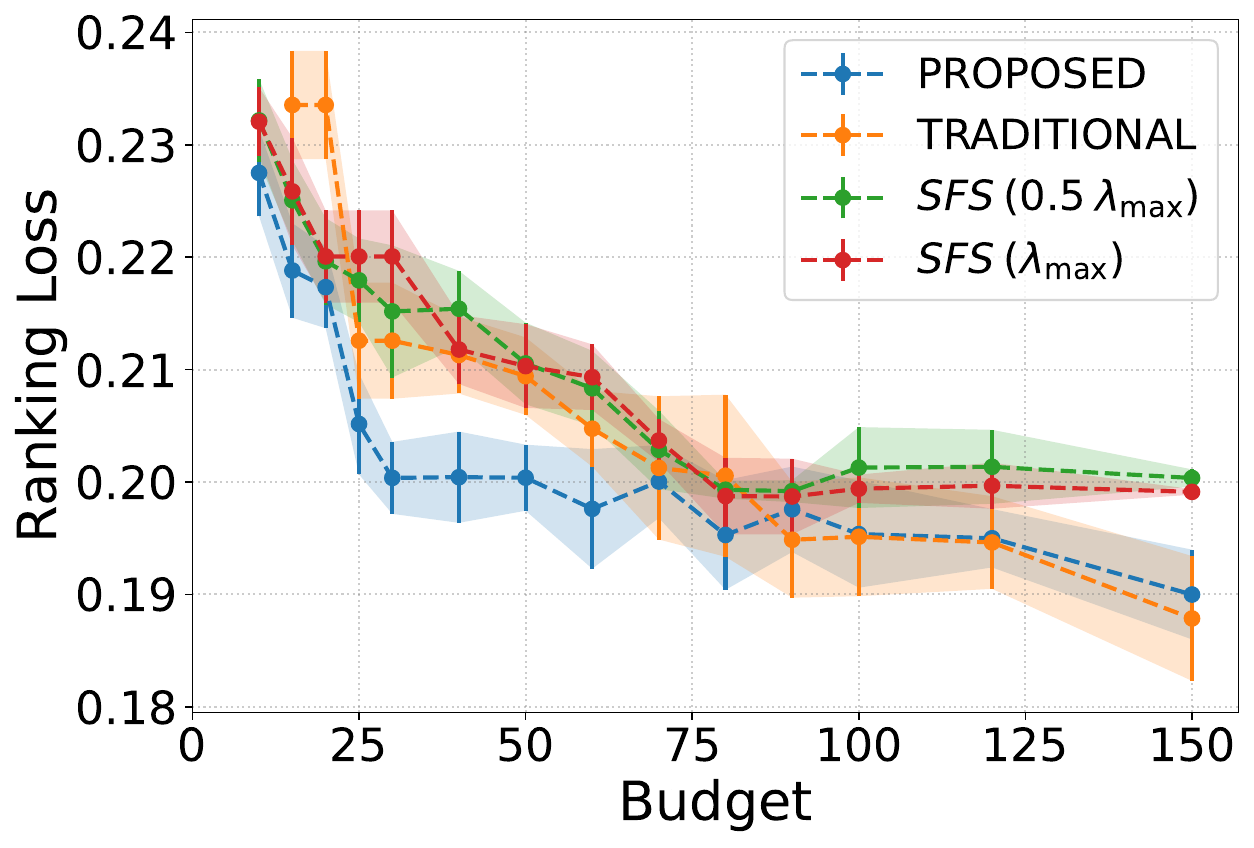}
      \includegraphics[width=0.33\textwidth]{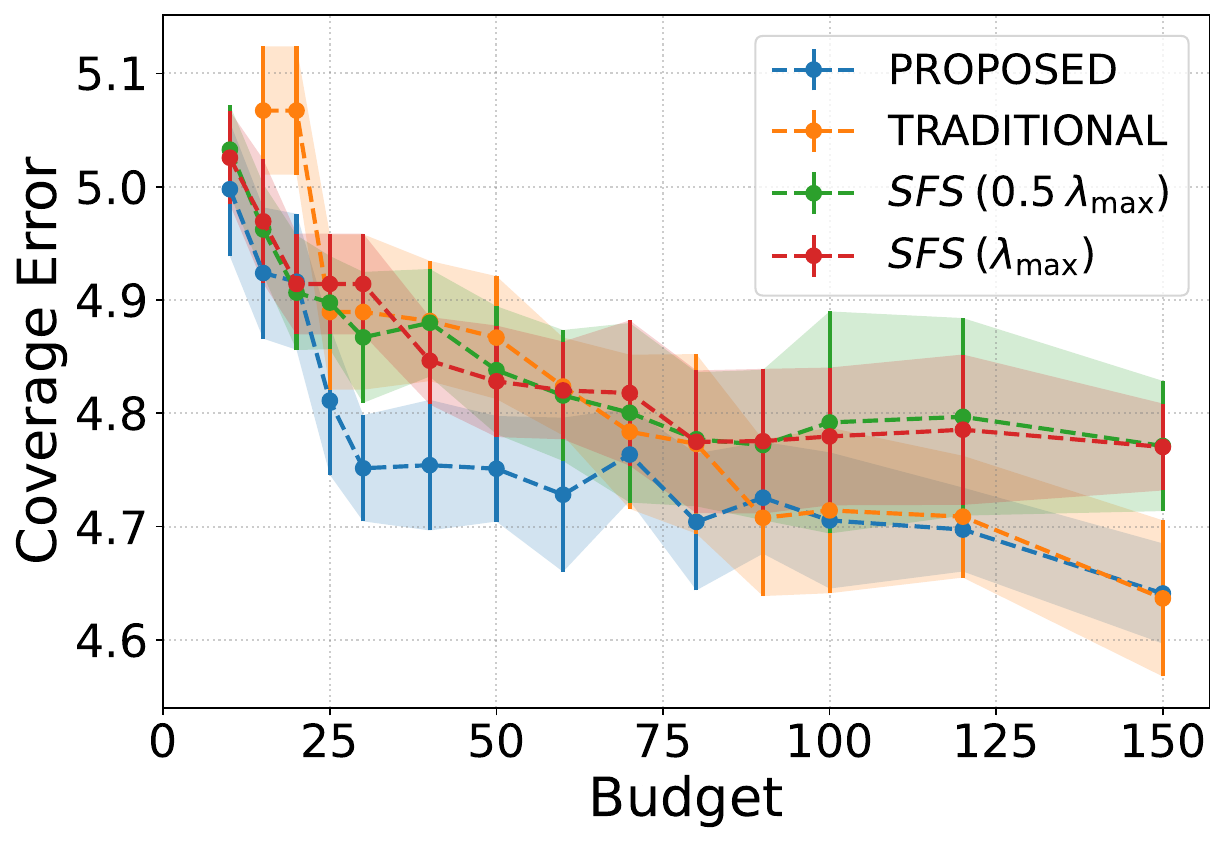}\\
      \includegraphics[width=0.33\textwidth]{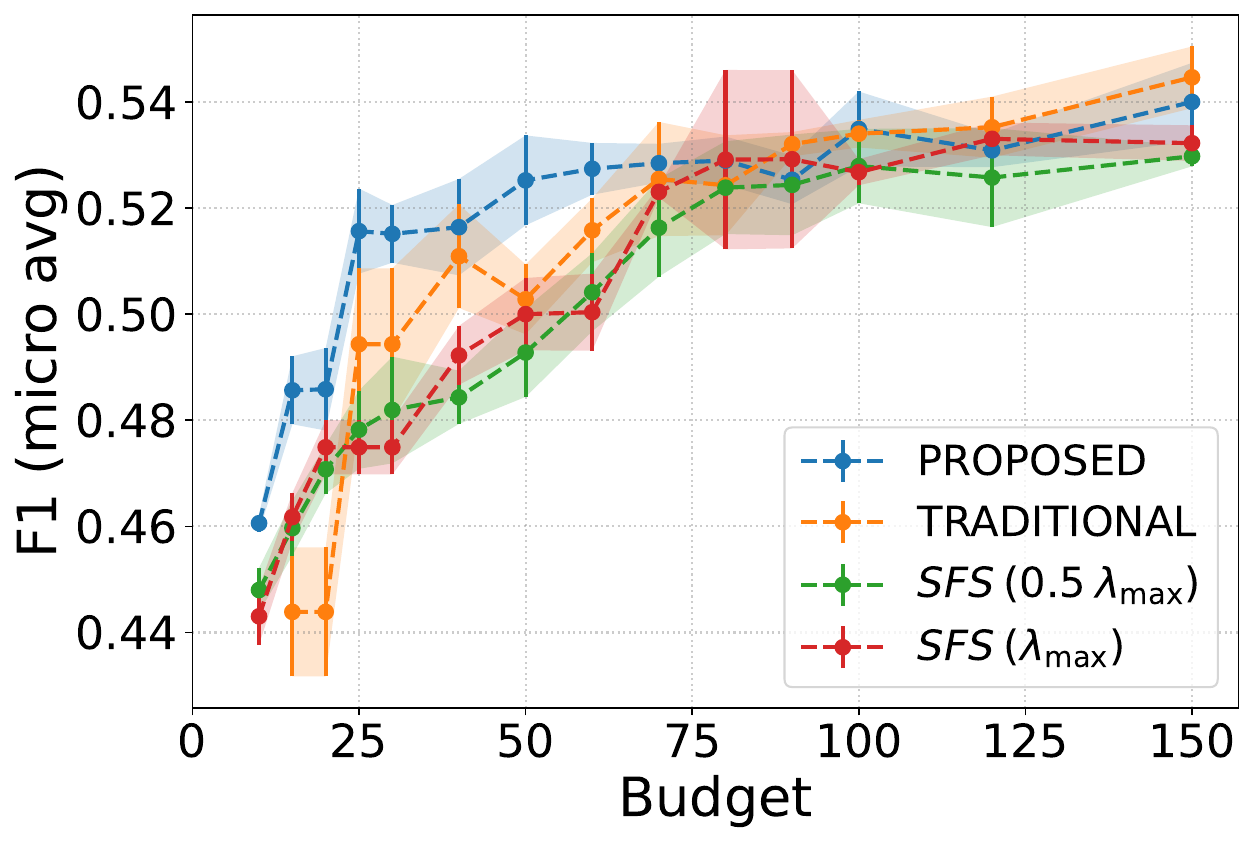}
      \includegraphics[width=0.33\textwidth]{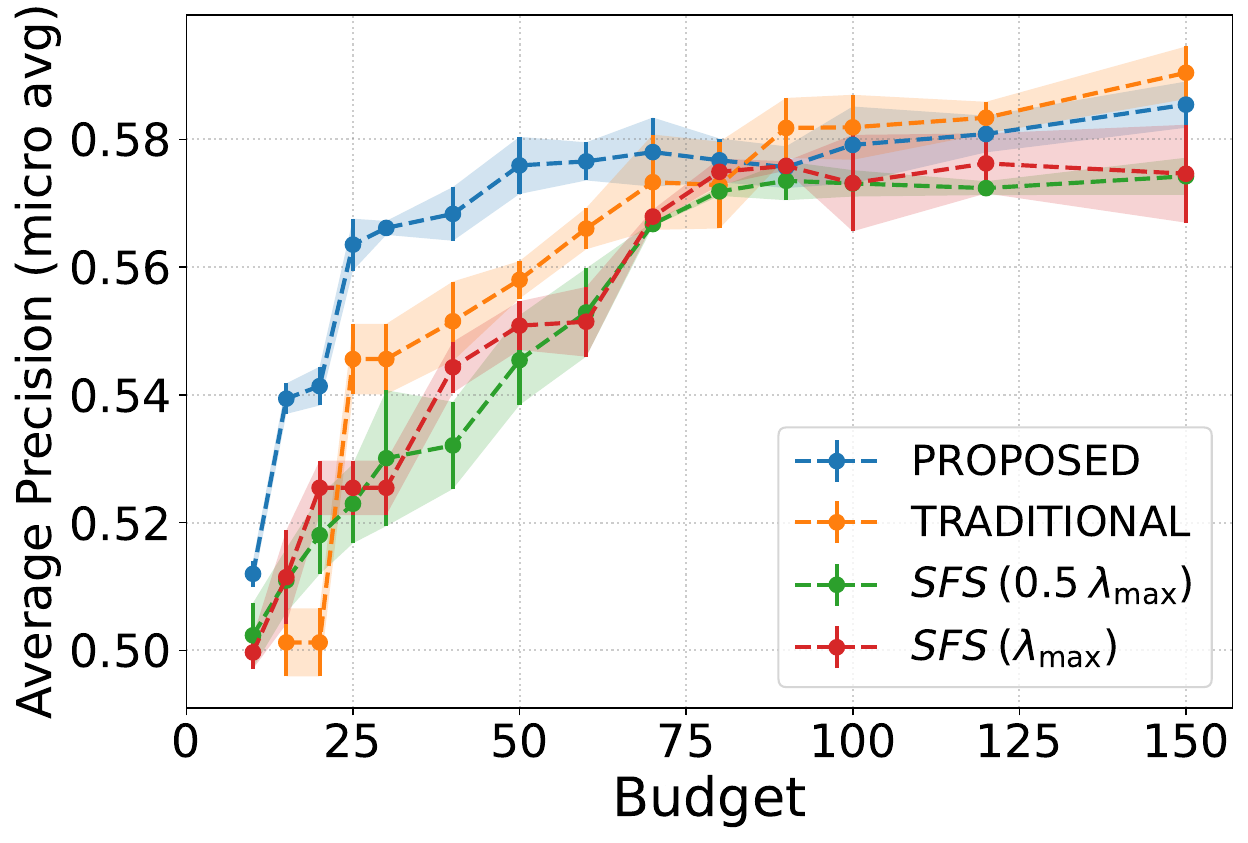}
      \includegraphics[width=0.33\textwidth]{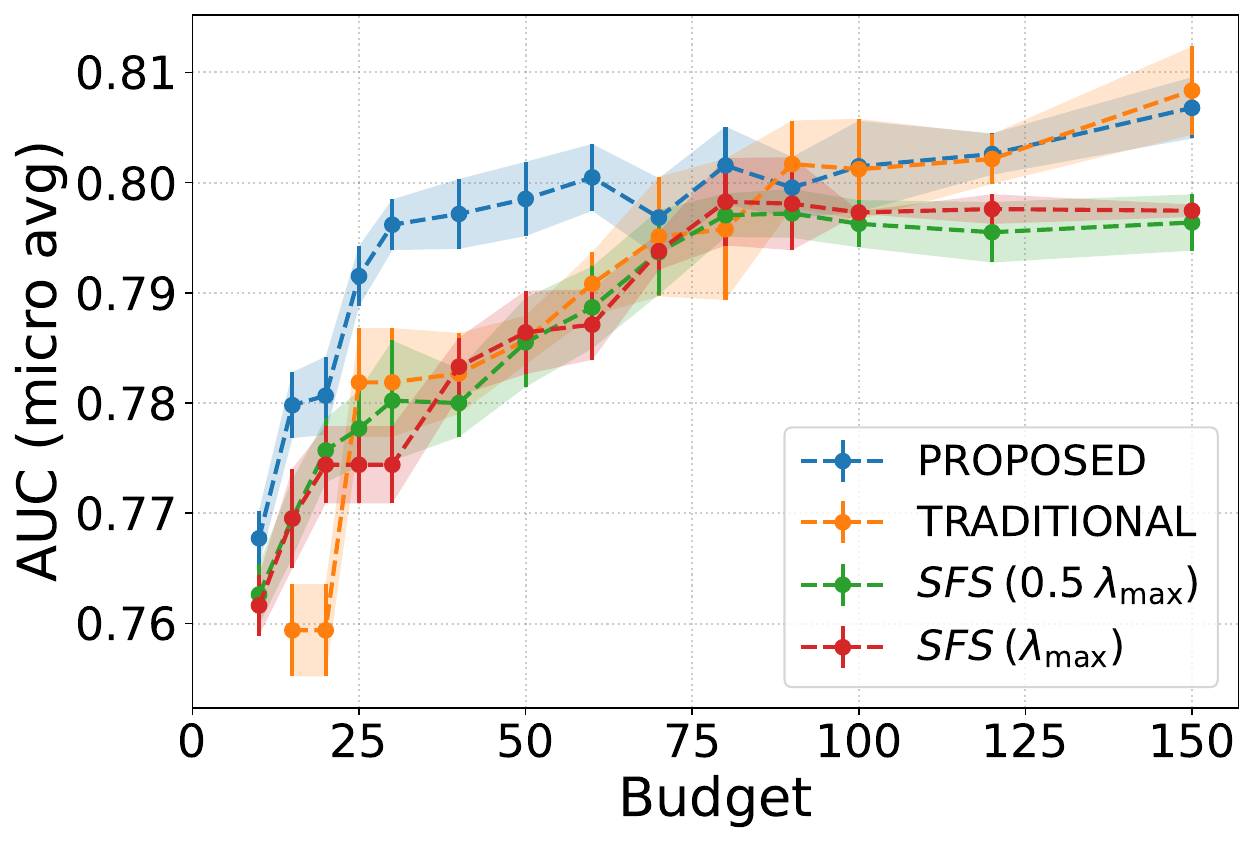}
    \end{tabular}
    \caption{Feature selection for the MIMIC dataset.}
    \label{fig:mimic2_results}
\end{figure*}

\subsection{Experiments on MIMIC dataset}
We performed experiments using the medical dataset MIMIC \citep{mimic2-dataset}, which includes data related to the medical status of patients admitted to intensive care units (ICUs). Each of the patients could be diagnosed with several diseases: hypertension ($65\%$ of the patients), diabetes ($31\%$), fluid ($31\%$), lipiod ($30\%$), kidney ($29\%$), copd ($22\%$), thyroid ($11\%$), hypotension ($10\%$), liver ($6\%$) and thrombosis ($5\%$). The features mainly represent the results of diagnostic tests. This dataset has been used in previous research, for a more comprehensive understanding of feature extraction and data cleaning procedures, we refer to \cite{Zufferey2015}.

The features and the cost of the MIMIC dataset are evaluated as assigned by domain experts in this field, respectively. A detailed list of these features and their respective costs can be found in \cite{Teisseyre2019}. In the present work, we expand on this methodology by categorizing features with similar characteristics into groups and assigning costs to these groups. Most of these groups combine four distinct statistics (mean, median, standard deviation, and range) of a particular medical parameter measured over a specific time frame. Another category, administrative data, includes fundamental patient information such as age, gender, or marital status, which can be obtained during a medical interview.

In Table~\ref{tab:mimic-summary-statistics2}, we present basic statistics that describe the dataset. We can see that the $305$ characteristics are grouped into $87$ groups. Furthermore, in Table~\ref{tab:groups}, we describe some of the most informative feature groups in detail. For example, administrative data is the cheapest group, with a cost of one and containing nine features. 
On the other hand, urine testing is one of the most expensive groups of features.

To evaluate the quality of the chosen feature subset, we use a multi-label $k$-Nearest Neighbors algorithm \citep{zhang2007ml}. Both FS and model training are performed on the training data ($80\%$), while the evaluation metrics are calculated on the validation set. We repeat the data split process five times to assess the variability of the results.

Figure~\ref{fig:mimic2_results} illustrates the selected evaluation measures of a model trained on feature sets chosen for different budgets. Additionally in Table \ref{tab:mimic_results_hamming_loss} we present the detailed values of various evaluation measures for budgets $B=10,25,120$.
Importantly, we see that the traditional method in the first steps selects features that do not allow them to fit into the low budget of $10$ (which is marked in the tables as '-').
The results indicate that too high value of the penalty parameter ($\lambda_{max}$) leads to the selection of cheap but not necessarily informative features, resulting in a lower prediction accuracy of the corresponding model.
When the budget is relatively low, below $75$, the proposed algorithm significantly outperforms other methods for every metric value. Moreover, as we move to a higher budget allocation, such as $100$, we observe that all methods present comparable performance, indicating that all informative features have already been selected. This shows that the proposed method can be recommended when the available budget is low.
\section{Conclusions}
In this work, we considered an unexplored problem of multi-label FS where the costs are assigned on the basis of feature groups. We proposed a novel two-step method that aims to reduce the total cost of the selected features. Using cost information in FS allows us to train a model with greater accuracy in a situation with a limited budget. Unlike SFS, there is no need to optimize the penalty parameter. Moreover, in SFS, the scale of cost values is important because costs are subtracted from the relevance score. In the proposed approach, we avoid this problem, which is an additional advantage.
Experiments indicate that the proposed method performs more effectively than a competing approach based on the SFS scheme. Importantly, the proposed method can be combined with different FS criteria. Future work may focus on using the described approach in other interesting medical applications where groups of features and costs assigned to them naturally arise.
Considering stopping methods other than those based on shadow features is also worth exploring.


\bibliography{mybibfile}

\begin{thebibliography}{37}
\providecommand{\natexlab}[1]{#1}
\providecommand{\url}[1]{\texttt{#1}}
\expandafter\ifx\csname urlstyle\endcsname\relax
  \providecommand{\doi}[1]{doi: #1}\else
  \providecommand{\doi}{doi: \begingroup \urlstyle{rm}\Url}\fi

\bibitem[Belghazi et~al.(2018)Belghazi, Baratin, Rajeshwar, Ozair, Bengio, Courville, and Hjelm]{Belghazietal2018}
M.~I. Belghazi, A.~Baratin, S.~Rajeshwar, S.~Ozair, Y.~Bengio, A.~Courville, and D.~Hjelm.
\newblock Mutual information neural estimation.
\newblock In \emph{Proceedings of the 35th International Conference on Machine Learning}, Proceedings of Machine Learning Research, pages 531--540, 2018.

\bibitem[Bogatinovski et~al.(2022)Bogatinovski, Todorovski, Džeroski, and Kocev]{BOGATINOVSKI2022}
J.~Bogatinovski, L.~Todorovski, S.~Džeroski, and D.~Kocev.
\newblock Comprehensive comparative study of multi-label classification methods.
\newblock \emph{Expert Systems with Applications}, 203:\penalty0 1--18, 2022.

\bibitem[Bol\'{o}n-Canedo et~al.(2014)Bol\'{o}n-Canedo, Porto-D\'{i}az, S\'{a}nchez-Maro\~{n}o, and Alonso-Betanzos]{BOLONCANEDO2014}
V.~Bol\'{o}n-Canedo, I.~Porto-D\'{i}az, N.~S\'{a}nchez-Maro\~{n}o, and A.~Alonso-Betanzos.
\newblock A framework for cost-based feature selection.
\newblock \emph{Pattern Recognition}, 47\penalty0 (7):\penalty0 2481--2489, 2014.

\bibitem[Brown et~al.(2012)Brown, Pocock, Zhao, and Luj\'{a}n]{Brown2012}
G.~Brown, A.~Pocock, M.~J. Zhao, and M.~Luj\'{a}n.
\newblock Conditional likelihood maximisation: A unifying framework for information theoretic feature selection.
\newblock \emph{Journal of Machine Learning Research}, 13\penalty0 (1):\penalty0 27--66, 2012.

\bibitem[Cover and Thomas(2006)]{Cover2006}
T.~M. Cover and J.~A. Thomas.
\newblock \emph{Elements of Information Theory (Wiley Series in Telecommunications and Signal Processing)}.
\newblock Wiley-Interscience, 2006.

\bibitem[Donsker and Varadhan(1975)]{DonskerVaradhan1975}
M.~D. Donsker and S.~R.~S. Varadhan.
\newblock Asymptotic evaluation of certain markov process expectations for large time, i.
\newblock \emph{Communications on Pure and Applied Mathematics}, 28\penalty0 (1):\penalty0 1--47, 1975.

\bibitem[Hall and Brenner(2008)]{HallBrenner2008}
E.~J. Hall and D.~J. Brenner.
\newblock Cancer risks from diagnostic radiology.
\newblock \emph{The British Journal of Radiology}, 81\penalty0 (965):\penalty0 362--378, 2008.

\bibitem[Han(1980)]{Han1980}
T.~S. Han.
\newblock Multiple mutual informations and multiple interactions in frequency data.
\newblock \emph{Information and Control}, 46\penalty0 (1):\penalty0 26--45, 1980.

\bibitem[Hastie et~al.(2009)Hastie, Tibshirani, and Friedman]{Hastieetal2009}
T.~Hastie, R.~Tibshirani, and J.~Friedman.
\newblock \emph{{T}he {E}lements of {S}tatistical {L}earning: {D}ata {M}ining, {I}nference and {P}rediction}.
\newblock Springer, 2009.
\newblock URL \url{http://www-stat.stanford.edu/~tibs/ElemStatLearn/}.

\bibitem[Jagdhuber et~al.(2020)Jagdhuber, Lang, Stenzl, Neuhaus, and Rahnenführer]{jagdhuber2020cost}
R.~Jagdhuber, M.~Lang, A.~Stenzl, J.~Neuhaus, and J.~Rahnenführer.
\newblock Cost-constrained feature selection in binary classification: adaptations for greedy forward selection and genetic algorithms.
\newblock \emph{BMC bioinformatics}, 21\penalty0 (1):\penalty0 1--21, 2020.

\bibitem[Kashef et~al.(2018)Kashef, Nezamabadi-pour, and Nikpour]{Kashefetal2018}
S.~Kashef, H.~Nezamabadi-pour, and B.~Nikpour.
\newblock {Multilabel feature selection: A comprehensive review and guiding experiments}.
\newblock \emph{{Wiley Interdisciplinary Reviews: Data Mining and Knowledge Discovery}}, 8\penalty0 (2):\penalty0 1--29, 2018.

\bibitem[Klonecki et~al.(2023{\natexlab{a}})Klonecki, Teisseyre, and Lee]{KloneckiTeisseyreLeeMDAI2023}
T.~Klonecki, P.~Teisseyre, and J.~Lee.
\newblock Cost-constrained group feature selection using information theory.
\newblock In \emph{Procceding of Modeling Decisions for Artificial Intelligence}, MDAI'23, pages 121--132, 2023{\natexlab{a}}.

\bibitem[Klonecki et~al.(2023{\natexlab{b}})Klonecki, Teisseyre, and Lee]{KloneckiTeisseyreLeePR2023}
T.~Klonecki, P.~Teisseyre, and J.~Lee.
\newblock Cost-constrained feature selection in multilabel classification using an information-theoretic approach.
\newblock \emph{Pattern Recognition}, 141:\penalty0 1--18, 2023{\natexlab{b}}.

\bibitem[Kursa and Rudnicki(2010)]{Kursa2011}
M.~Kursa and W.~Rudnicki.
\newblock Feature selection with the boruta package.
\newblock \emph{Journal of Statistical Software}, 36\penalty0 (11):\penalty0 1--13, 2010.

\bibitem[Lee and Kim(2012)]{LeeKim2012}
J.~Lee and D.~Kim.
\newblock Approximating mutual information for multi-label feature selection.
\newblock \emph{Electronics Letters}, 48:\penalty0 929--930, 2012.

\bibitem[Lee and Kim(2015)]{Lee2015mutual}
J.~Lee and D.~Kim.
\newblock Mutual information-based multi-label feature selection using interaction information.
\newblock \emph{Expert Systems with Applications}, 42\penalty0 (4):\penalty0 2013--2025, 2015.

\bibitem[Lee and Kim(2017)]{LeeKim2017}
J.~Lee and D.~Kim.
\newblock {SCLS: Multi-label feature selection based on scalable criterion for large label set}.
\newblock \emph{Pattern Recognition}, 66:\penalty0 342 -- 352, 2017.

\bibitem[Lin and Tang(2006)]{Lin2006}
D.~Lin and X.~Tang.
\newblock Conditional infomax learning: An integrated framework for feature extraction and fusion.
\newblock In \emph{Proceedings of the 9th European Conference on Computer Vision}, ECCV'06, pages 68--82, 2006.

\bibitem[Long et~al.(2021)Long, Qian, Wang, and Shu]{long2021cost}
X.~Long, W.~Qian, Y.~Wang, and W.~Shu.
\newblock Cost-sensitive feature selection on multi-label data via neighborhood granularity and label enhancement.
\newblock \emph{Applied Intelligence}, 51\penalty0 (4):\penalty0 2210--2232, 2021.

\bibitem[Meyer et~al.(2008)Meyer, Schretter, and Bontempi]{Meyer2008}
P.~Meyer, C.~Schretter, and G.~Bontempi.
\newblock {Information-theoretic feature selection in microarray data using variable complementarity}.
\newblock \emph{IEEE Journal of Selected Topics in Signal Processing}, 2\penalty0 (3):\penalty0 261--274, 2008.

\bibitem[Mielniczuk and Teisseyre(2018)]{MielniczukTeisseyre2018}
J.~Mielniczuk and P.~Teisseyre.
\newblock A deeper look at two concepts of measuring gene-gene interactions: logistic regression and interaction information revisited.
\newblock \emph{Genetic Epidemiology}, 42\penalty0 (2):\penalty0 187--200, 2018.

\bibitem[Molnar(2022)]{Molnar2022}
C.~Molnar.
\newblock \emph{Interpretable Machine Learning}.
\newblock 2 edition, 2022.
\newblock URL \url{https://christophm.github.io/interpretable-ml-book}.

\bibitem[Pacl{\'i}k et~al.(2002)Pacl{\'i}k, Duin, van Kempen, and Kohlus]{Pacliketal2002}
P.~Pacl{\'i}k, R.~Duin, G.~van Kempen, and R.~Kohlus.
\newblock On feature selection with measurement cost and grouped features.
\newblock In \emph{Proceedings of Joint IAPR International Workshops on Statistical Techniques in Pattern Recognition and Structural and Syntactic Pattern Recognition}, pages 461--469, 2002.

\bibitem[Paninski(2003)]{Paninski2003}
L.~Paninski.
\newblock Estimation of entropy and mutual information.
\newblock \emph{Neural Computation}, 15\penalty0 (6):\penalty0 1191--1253, 2003.

\bibitem[Saeed et~al.(2011)Saeed, Villarroel, Reisner, Clifford, Lehman, Moody, Heldt, Kyaw, Moody, and Mark]{mimic2-dataset}
M.~Saeed, M.~Villarroel, A.~T. Reisner, G.~Clifford, L.~W. Lehman, G.~Moody, T.~Heldt, T.~H. Kyaw, B.~Moody, and R.~G. Mark.
\newblock {Multiparameter intelligent monitoring in intensive care II: A public-access intensive care unit database}.
\newblock \emph{Critical Care Medicine}, 39\penalty0 (5):\penalty0 952--960, 2011.

\bibitem[Seo et~al.(2019)Seo, Kim, and Lee]{seo2019generalized}
W.~Seo, D.~Kim, and J.~Lee.
\newblock Generalized information-theoretic criterion for multi-label feature selection.
\newblock \emph{IEEE Access}, 7:\penalty0 122854--122863, 2019.

\bibitem[Teisseyre and Klonecki(2021)]{TeisseyreKlonecki2021}
P.~Teisseyre and T.~Klonecki.
\newblock Controlling costs in feature selection: Information theoretic approach.
\newblock In \emph{Proceedings of the Internation Conference of Computational Science}, ICCS, pages 483--496, 2021.

\bibitem[Teisseyre and Lee(2023)]{TEISSEYRELEE2023}
P.~Teisseyre and J.~Lee.
\newblock Multilabel all-relevant feature selection using lower bounds of conditional mutual information.
\newblock \emph{Expert Systems with Applications}, 216:\penalty0 1--19, 2023.

\bibitem[Teisseyre et~al.(2019)Teisseyre, Zufferey, and S{\l}omka]{Teisseyre2019}
P.~Teisseyre, D.~Zufferey, and M.~S{\l}omka.
\newblock {Cost-sensitive classifier chains: Selecting low-cost features in multi-label classification}.
\newblock \emph{Pattern Recognition}, 86:\penalty0 290--319, 2019.

\bibitem[Vegting et~al.(2012)Vegting, van Beneden, Kramer, Abel~Thijs, Kostense, and Nanayakkara]{Vegting2012}
I.~L. Vegting, M.~van Beneden, M.~H.~H. Kramer, A.~Abel~Thijs, P.~J. Kostense, and P.~W.~B. Nanayakkara.
\newblock How to save costs by reducing unnecessary testing: Lean thinking in clinical practice.
\newblock \emph{European Journal of Internal Medicine}, 23\penalty0 (1):\penalty0 70 -- 75, 2012.

\bibitem[Xu et~al.(2014)Xu, Kusner, Weinberger, Chen, and Chapelle]{Xuetal2014}
Z.~E. Xu, M.~J. Kusner, K.~Q. Weinberger, M.~Chen, and O.~Chapelle.
\newblock Classifier cascades and trees for minimizing feature evaluation cost.
\newblock \emph{Journal of Machine Learning Research}, 15\penalty0 (1):\penalty0 2113--2144, 2014.

\bibitem[Yang and Moody(1999)]{Yang1999}
H.~H. Yang and J.~Moody.
\newblock {Data visualization and feature selection: new algorithms for nongaussian data}.
\newblock \emph{Advances in Neural Information Processing Systems}, 12:\penalty0 687--693, 1999.

\bibitem[Yuan and Lin(2006)]{YuanLin2006}
M.~Yuan and Y.~Lin.
\newblock Model selection and estimation in regression with grouped variables.
\newblock \emph{Journal of the Royal Statistical Society: Series B (Statistical Methodology)}, 68\penalty0 (1):\penalty0 49--67, 2006.

\bibitem[Zhang and Zhou(2007)]{zhang2007ml}
M.~Zhang and Z.~Zhou.
\newblock Ml-knn: A lazy learning approach to multi-label learning.
\newblock \emph{Pattern recognition}, 40\penalty0 (7):\penalty0 2038--2048, 2007.

\bibitem[Zhang et~al.(2019)Zhang, Liu, and Gao]{zhang2019distinguishing}
P.~Zhang, G.~Liu, and W.~Gao.
\newblock Distinguishing two types of labels for multi-label feature selection.
\newblock \emph{Pattern Recognition}, 95:\penalty0 72--82, 2019.

\bibitem[Zhou et~al.(2016)Zhou, Zhou, and Li]{ZHOU2016}
Q.~Zhou, H.~Zhou, and T.~Li.
\newblock Cost-sensitive feature selection using random forest: Selecting low-cost subsets of informative features.
\newblock \emph{Knowledge-Based Systems}, 95:\penalty0 1--11, 2016.

\bibitem[Zufferey et~al.(2015)Zufferey, Hofer, Hennebert, Schumacher, Ingold, and Bromuri]{Zufferey2015}
D.~Zufferey, T.~Hofer, J.~Hennebert, M.~Schumacher, R.~Ingold, and S.~Bromuri.
\newblock {Performance comparison of multi-label learning algorithms on clinical data for chronic diseases}.
\newblock \emph{Computers in Biology and Medicine}, 65:\penalty0 34--43, 2015.

\end{thebibliography}

\end{document}